\pdfoutput=1

\documentclass[11pt]{article}

\usepackage{emnlp2021}

\usepackage{times}
\usepackage{latexsym}
\usepackage{amsmath}
\usepackage{amsthm}
\usepackage{amssymb}

\newtheorem{claim}{Claim}
\newtheorem*{theorem*}{Theorem}

\newtheorem*{lemma*}{Lemma}

\DeclareMathOperator*{\Ex}{\text{E}}

\DeclareMathOperator{\Tr}{Tr}

\DeclareMathOperator{\DDC}{\textnormal{DDC}}

\usepackage[T1]{fontenc}

\usepackage[utf8]{inputenc}

\usepackage{microtype}

\usepackage{amsmath}
\usepackage{graphicx}
\usepackage{siunitx}
\usepackage{float}
\usepackage{subcaption}
\usepackage{csquotes}
\usepackage{calc}
\usepackage{booktabs}
\usepackage{amsmath}
\usepackage{amssymb}
\usepackage{bbm}
\usepackage[most]{tcolorbox}
\usepackage{subfiles}
\usepackage{balance}

\title{Comparing Text Representations: A Theory-Driven Approach}

\author{Gregory Yauney \\ Cornell University \\  \texttt{gyauney@cs.cornell.edu}
        \And
        David Mimno \\ Cornell University \\ \texttt{mimno@cornell.edu}}

\begin{document}
\maketitle
\begin{abstract}
Much of the progress in contemporary NLP has come from learning representations, such as masked language model (MLM) contextual embeddings, that turn challenging problems into simple classification tasks. But how do we quantify and explain this effect? We adapt general tools from computational learning theory to fit the specific characteristics of text datasets and present a method to evaluate the compatibility between representations and tasks. Even though many tasks can be easily solved with simple bag-of-words (BOW) representations, BOW does poorly on hard natural language inference tasks. For one such task we find that BOW cannot distinguish between real and randomized labelings, while pre-trained MLM representations show 72x greater distinction between real and random labelings than BOW. This method provides a calibrated, quantitative measure of the difficulty of a classification-based NLP task, enabling comparisons between representations without requiring empirical evaluations that may be sensitive to initializations and hyperparameters. The method provides a fresh perspective on the patterns in a dataset and the alignment of those patterns with specific labels.
\end{abstract}

\begin{figure*}[!t]
         \centering
         \includegraphics[width=\textwidth]{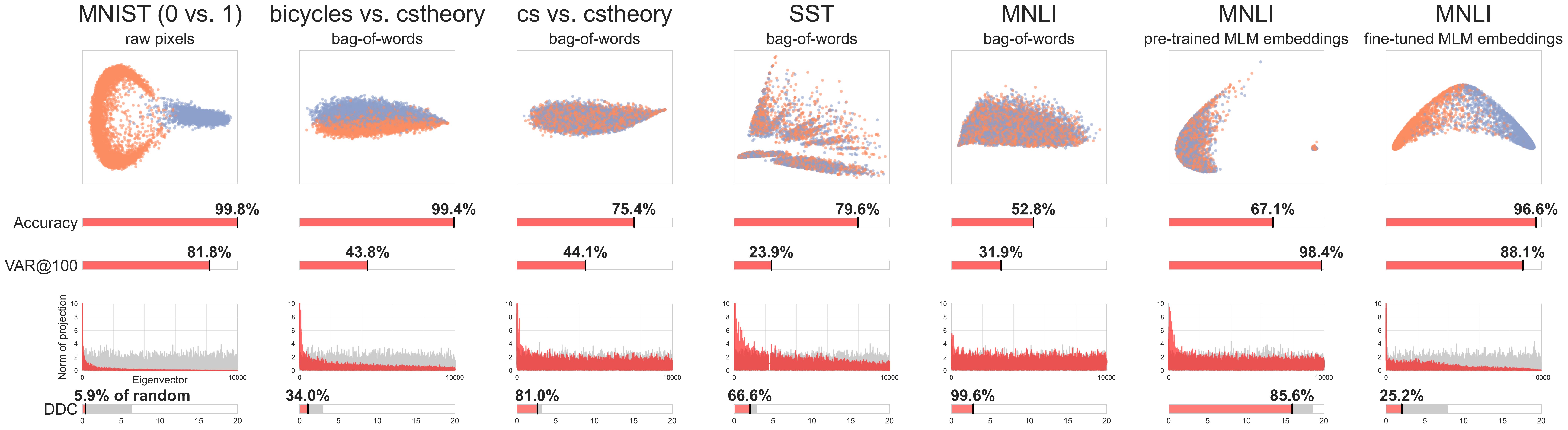}
        \caption{DDC shows how fine-tuning MLM embeddings turns a hard problem (NLI) into an easy problem. At the top, we show 2D PCA plots of the DDC Gram matrix for five classification problems, from easy (MNIST) to hard (MNLI), along with two MLM-based representations of MNLI. We then show empirical dev-set accuracy and the fraction of variance explained by the first 100 eigenvectors. DDC measures the projection of the labels onto each eigenvector (red histogram), scaled by the inverse of the eigenvalue. The bottom row shows DDC as a proportion of DDC for random labels (gray histogram). Fine-tuned embeddings turn MNLI from a task that is indistinguishable from random guessing into one that is as easy as telling if a post is about bicycles or CS theory.}
        \label{fig:pca-projection-dashboard}
\end{figure*}

\section{Introduction}


A common theme in contemporary machine learning is representation learning: a task that is complicated and difficult can be transformed into a simple classification task by filtering the input through a deep neural network.
For example, we know empirically that it is difficult to train a classifier for natural language inference (NLI)---determining whether a sentence logically entails another---using bag-of-words features as inputs, but training the same type of classifier on the output of a pre-trained masked language model (MLM) results in much better performance \cite{liu2019roberta}.
Fine-tuned representations do even better.
%
But why is switching from raw text features to MLM contextual embeddings so successful for downstream classification tasks?
Probing strategies can map the syntactic and semantic information encoded in contextual embeddings \cite{rogers2020primer}, but it remains difficult to compare embeddings for classification beyond simply measuring differences in task accuracy.
What makes a given input representation easier or harder to map to a specific set of labels?

In this work we adapt a tool from computational learning theory, data-dependent complexity (DDC) \cite{arora2019fine}, to analyze the properties of a given text representation for a classification task.
Given input vectors and an output labeling, DDC provides theoretical bounds on the performance of an idealized two-layer ReLU network.
At first, this method may not seem applicable to contemporary NLP: this network is simple enough to prove bounds about, but does not even begin to match the complexity of current Transformer-based models.
Although there has been work to extend the analysis of \citet{arora2019fine} to more complicated networks \cite{allen2018learning}, the simple network \textit{is} a good approximation for a task-specific classifier head. We therefore take a different approach, and use DDC to measure the properties of \textit{representations} learned by networks, not the networks themselves.
This approach does not require training any actual classification models, and is therefore not dependent on hyperparameter settings, initializations, or stochastic gradient descent.

Quantifying the relationship between representations and labels has important practical impacts for NLP.
Text data has long been known to differ from other kinds of data in its high dimensionality and sparsity \cite{joachims2001statistical}.
We analyze the difficulty of NLP tasks with respect to two distinct factors: the complexity of patterns in the dataset, and the alignment of those patterns with the labels.
Better ways to analyze relationships between representations and labels may enable us to better handle problems with datasets, such as ``shortcut'' features that are spuriously correlated with labels \cite{gururangan2018annotation, thompson2018authorless, geirhos2020shortcut,le2020adversarial}.



Our contributions are the following.
First, we identify and address several practical issues in applying DDC for data-label alignment in text classification problems, including better comparisons to ``null'' distributions to handle harder classification problems and enable comparisons across distinct representations.
Second, we define three evaluation patterns that provide calibrated feedback for data curation and modeling choices: 
For a given representation (such as MLM embeddings), are some labelings more or less compatible with that representation?
For a given target labeling, is one or another representation more effective?
How can we measure and explain the difficulty of text classification problems between datasets?
Third, we provide case studies for each of these usages. In particular, we use our method to quantify the difference between various localist and neural representations of NLI datasets for classification, identifying differences between datasets and explaining the difference between MLM embeddings and simpler representations.\footnote{Code is available at: \texttt{\url{https://github.com/gyauney/data-label-alignment}}}

\section{Data-Dependent Complexity}

Data-dependent complexity \cite{arora2019fine} combines measurements of two properties of a binary-labeled dataset: the strength of patterns in the input data and the alignment of the output labels with those patterns.
Patterns in data are captured by a pairwise document-similarity (Gram) matrix.
An eigendecomposition is a representation of a matrix in terms of a set of basis vectors (the eigenvectors) and the relative importance of those vectors (the eigenvalues).
If we can reconstruct the original matrix with high accuracy using only a few eigenvectors, their corresponding eigenvalues will be large relative to the remaining eigenvalues.
A matrix with more complicated structure will have a more uniform sequence of eigenvalues.
DDC measures the projections of the label vector onto each eigenvector, scaled by the inverse of the corresponding eigenvalue. A label vector that can be reconstructed with high accuracy using only the eigenvectors with the largest eigenvalues will therefore have low DDC, while a label vector that can only be reconstructed using many eigenvectors with small eigenvalues will have high DDC.

\paragraph{Motivating examples.}
Figure~\ref{fig:pca-projection-dashboard} shows PCA plots
of Gram matrices for five datasets. Each point represents a document, colored by its label.
As an informal intuition, if we can linearly separate the classes using this 2D projection, the dataset will definitely have low DDC.
DDC can provide a perspective on difficulty beyond just comparing accuracy, especially when using a powerful classifier, where accuracies can be nearly perfect even for complicated problems. The MNIST digit classification dataset \cite{lecun1998gradient} and an intentionally easy text dataset (distinguishing posts from Stack Exchange forums on bicycles and CS theory) are two tasks
on which the simple network studied by \citet{arora2019fine} achieves high accuracy: 99.8\% and 99.4\%, respectively.
MNIST is relatively simple: 81.8\% of the variance is explained by the first 100 eigenvectors.
{ DDC is low for MNIST because the dominant pattern of the dataset aligns with the labels. Since the eigenvalues decay quickly, their inverses increase quickly, but the label vector projects only onto the top few eigenvectors. Any \textit{other} label vector would likely have much higher DDC.}
Bicycles vs. CS theory, while also simple, is more complicated from an eigenvector perspective, with only 43.8\% of variance explained.
Even though the eigenvalues decay more slowly, the labels project onto enough lower-ranked eigenvectors that DDC is higher than in MNIST.
Both MNIST and Bicycles vs. CS theory are easy in this operational sense, but DDC nevertheless shows there is a meaningful difference when accuracy saturates: more complicated patterns must be fit in order to learn the Bicycles vs. CS theory task to high accuracy.

MNLI \cite{N18-1101} is a much harder text classification dataset. We get lower accuracy for simple networks trained on two representations, bag-of-words and pre-trained MLM embeddings.
{ In this case differences in accuracy are more informative, but DDC still provides additional information, provided that we contextualize it.
Comparing raw DDC values seems to contradict accuracy: the task with a bag-of-words representation has a DDC of 2.8 while pre-trained embeddings produce a DDC of 15.9.
In this case, DDC is higher not because it is putting more weight on lower-ranked eigenvectors (the opposite is true), but because the eigenvalues for the pre-trained embeddings drop more quickly: 98.4\% of variance is explained by the first 100 eigenvectors.
To account for this difference, it is necessary to calibrate DDC by normalizing relative to the DDC of random labelings.
The relative gap between the DDC of a real labeling and DDC for a random labeling is much larger for MNLI under pre-trained MLM embeddings: BOW is indistinguishable from random labels (as the near-50\% accuracy suggests) while pre-trained embeddings distinguish MNLI labels from random labels at above-random performance.}
MNLI using fine-tuned MLM embeddings, finally, has both low eigenvector complexity (88.1\% variance explained) and allows for almost perfect classification accuracy with low relative DDC.

\paragraph{From dataset to data-dependent complexity.}
The complexity of classification tasks is studied in computational learning theory.
Rademacher complexity goes beyond the worst-case characterization of VC-Dimension to measure 
the gap between how well a family of classifiers can fit arbitrary labels for a fixed set of inputs and how well the classifier fits the given real labels of those inputs \cite{shalev2014understanding}.\footnote{Note that we are not referring to \textit{linguistic} complexity of text, as in \citet{bentz2016comparison, bentz2017entropy, gutierrez2020productivity}.}
In this work, we turn this around and compare the capacity of a fixed classifier to fit arbitrary labels for different \textit{input representations}, including multiple representations of the same data.
{Downsides of calculating Rademacher complexity directly are 1) in the general setting it requires taking a supremum over the family of classifiers and 2) it will trivially saturate if the dataset is smaller than the classifier's VC-Dimension.}
\citet{arora2019fine}
show that for large-width two-layer ReLU networks, the projections of labels onto the eigenvectors of a Gram matrix govern both generalization error and the rate of convergence of SGD.
Similar spectral analysis of Gram matrices has long been used in kernel learning \cite{cristianini2001kernel}.

The foundation of the \citet{arora2019fine} data-dependent complexity measure is the Gram matrix that measures similarity between documents.
As a model we use an overparameterized two-layer ReLU network to ensure comparability with prior work.
Let $\mathbf{x}_i$ be the $\ell_2$-normalized representation of the $i^\text{th}$ document out of $n$, and $y_i$ be the label for that document.
We construct this matrix of pairwise document similarities under a ReLU kernel, where the similarity between documents $\mathbf{x}_i$ and $\mathbf{x}_j$ is
\begin{align}
    \mathbf{H}_{ij}^\infty & = \frac{\mathbf{x}_i^\intercal \mathbf{x}_j \, \big(\pi - \arccos (\mathbf{x}_i^\intercal \mathbf{x}_j )\big) }{2 \pi}.
\end{align}
{This kernel is discussed in more detail in \citet{arora2019fine} and \citet{xie2017diverse}.}
Letting $\mathbf{Q\Lambda Q}^\intercal$ be the eigendecomposition of $\mathbf{H}^\infty$ and using the identity that $(\mathbf{Q}\mathbf{\Lambda}\mathbf{Q}^\intercal)^{-1} = \mathbf{Q} \mathbf{\Lambda}^{-1} \mathbf{Q}^\intercal$, DDC is
\begin{align}
\text{DDC} & = \sqrt{\frac{2 \mathbf{y}^\intercal (\mathbf{H}^\infty)^{-1} \mathbf{y}}{n}} \\
& = \sqrt{\frac{2 (\mathbf{y}^\intercal \mathbf{Q})\mathbf{\Lambda}^{-1}(\mathbf{Q}^\intercal \mathbf{y})}{n}}. \label{eqn:eigenvectors}
\end{align}
We refer to $\mathbf{y}^\intercal \mathbf{Q}$ as the \textit{projections} of a label vector onto the eigenvectors of the Gram matrix.
We call a task's labeling $\mathbf{y}$ the \textit{real} labeling.
\citet{arora2019fine} show that lower DDC implies faster convergence and lower generalization error.

\begin{figure}[!t]
         \centering
         \includegraphics[width=\columnwidth]{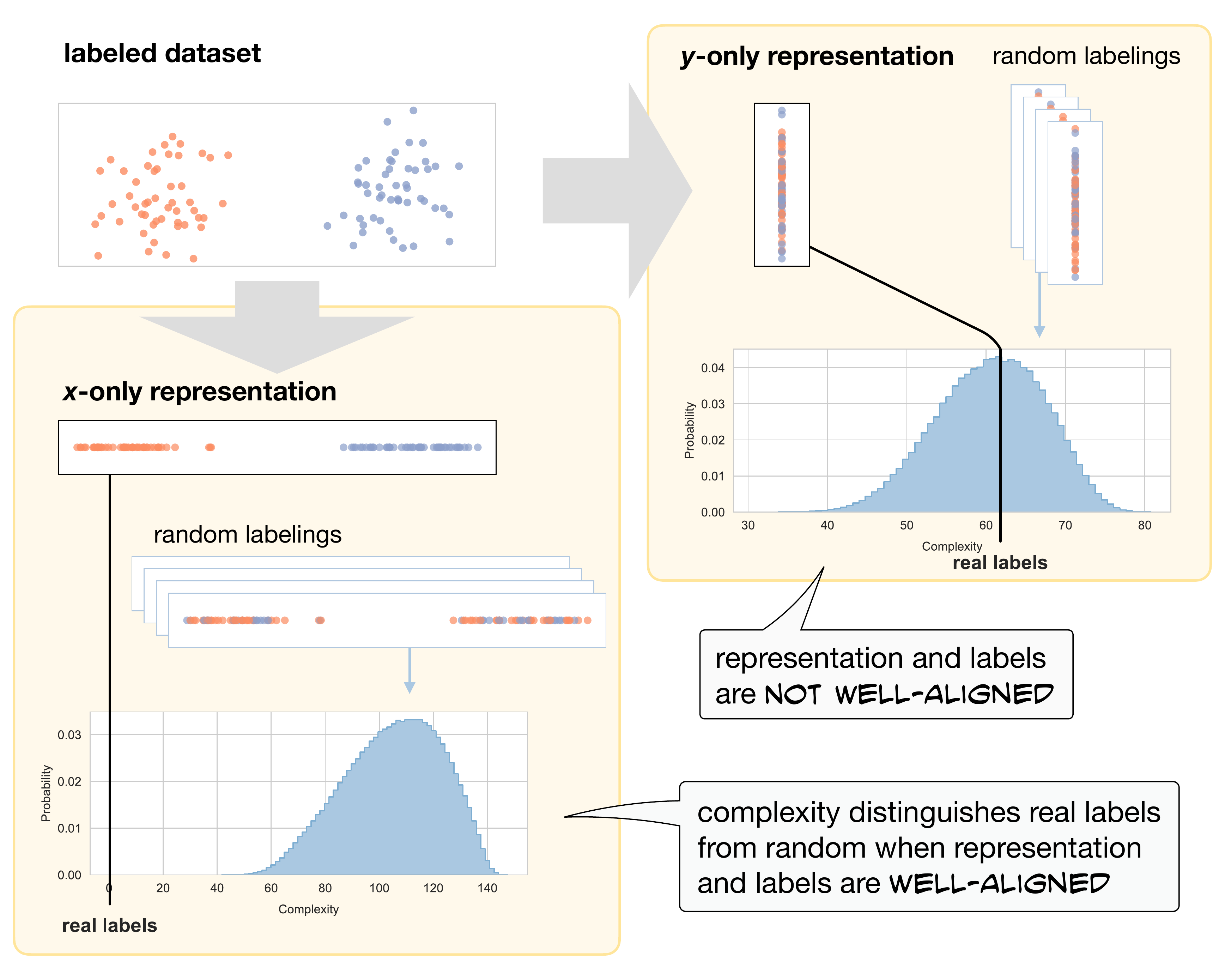}
        \caption{Intuition: data-label alignment as a way to compare representations. For this labeling of these points, the $x$-only representation is well-aligned with the labeling, as there is a large gap between the DDCs of the real labeling and random labelings. The $y$-only representation does not distinguish between real and random labelings. 
        }
        \label{fig:overview}
\end{figure}

\section{Making Data-Dependent Complexity a Practical Tool for Text Data}
\label{sect:practical-tool}


Unlike previous work, our goal is not to prove theoretical bounds on neural networks, but to evaluate the theoretical properties of different \textit{representations} of datasets.
Rather than compare DDC across representations directly, { data-label alignment takes inspiration from Rademacher complexity and} compares the gap between DDC of real and random labelings to account for different embedding spaces.
Figure~\ref{fig:overview} provides a simplified view of this approach, showing the DDC of real labels relative to DDC for random labelings for two trivial representations of a synthetic dataset.
We also find that several additional adaptations from \citet{arora2019fine} are required to make DDC an effective tool for text datasets.
Subsampling large datasets can provide good approximations with reduced computational cost,
and we show that
duplicate documents, which are more likely in text than in images, can significantly affect DDC if not handled.



\paragraph{Difficult datasets require distributions of random labels.}
We recommend comparing the DDC of the real labeling to the distribution of DDCs from many random labelings.
For easier tasks, there is wide separation between DDC values for real and random labels,
as \citet{arora2019fine} show for the MNIST 0 vs. 1 task.
For more difficult tasks, comparing the real labeling to only one random labeling could result in wildly different answers.
In a sample from the MNLI dataset with text represented as bags-of-words, for example, 30\% of random labelings had \textit{lower} complexity than the real labelings (Figure~\ref{fig:mnli-bow-histogram}).\footnote{Every MNLI subsample we examined had at least some random labelings with lower DDC than the real labeling.} 
Appendix~\ref{sec:app-bound} gives a bound on the number of random labelings required to get an accurate estimate of the expected DDC of a random labeling that is mainly determined by the gap between the inverses of the largest and smallest eigenvalues of the Gram matrix.
In our experiments, the number of random labelings ranges from a few hundred to several thousand; once eigenvectors have been calculated these are easily evaluated. We refer to the sampled estimate of expected DDC over random labelings as $\Ex\!\left[\text{DDC}\right]$.

\begin{figure}[!t]
         \centering
         \includegraphics[width=\columnwidth]{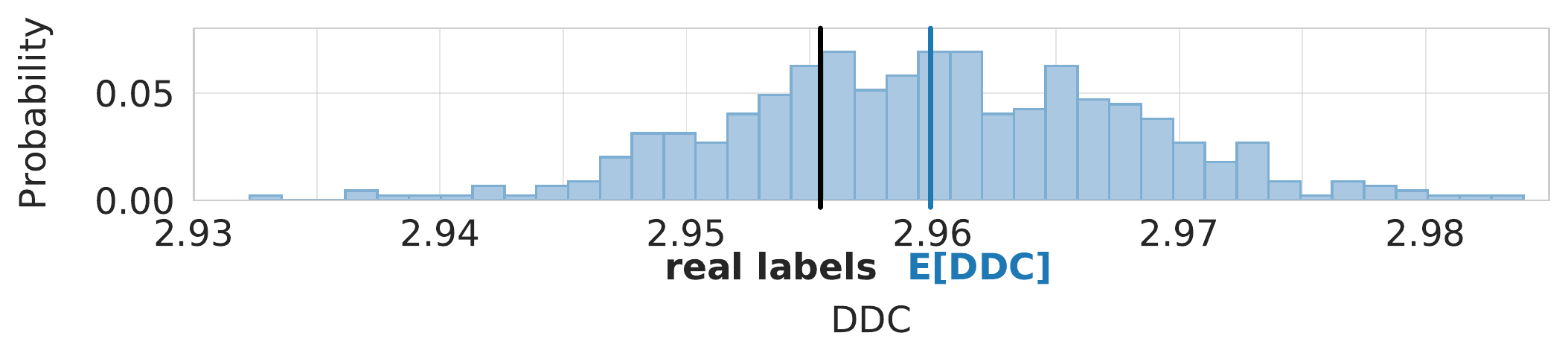}
        \caption{The DDC of nearly 30\% of sampled random labelings is less than that of the real labeling for a sample of the MNLI dataset represented by bags-of-words.}
        \label{fig:mnli-bow-histogram}
\end{figure}

\begin{figure}[!t]
         \centering
         \includegraphics[width=\columnwidth]{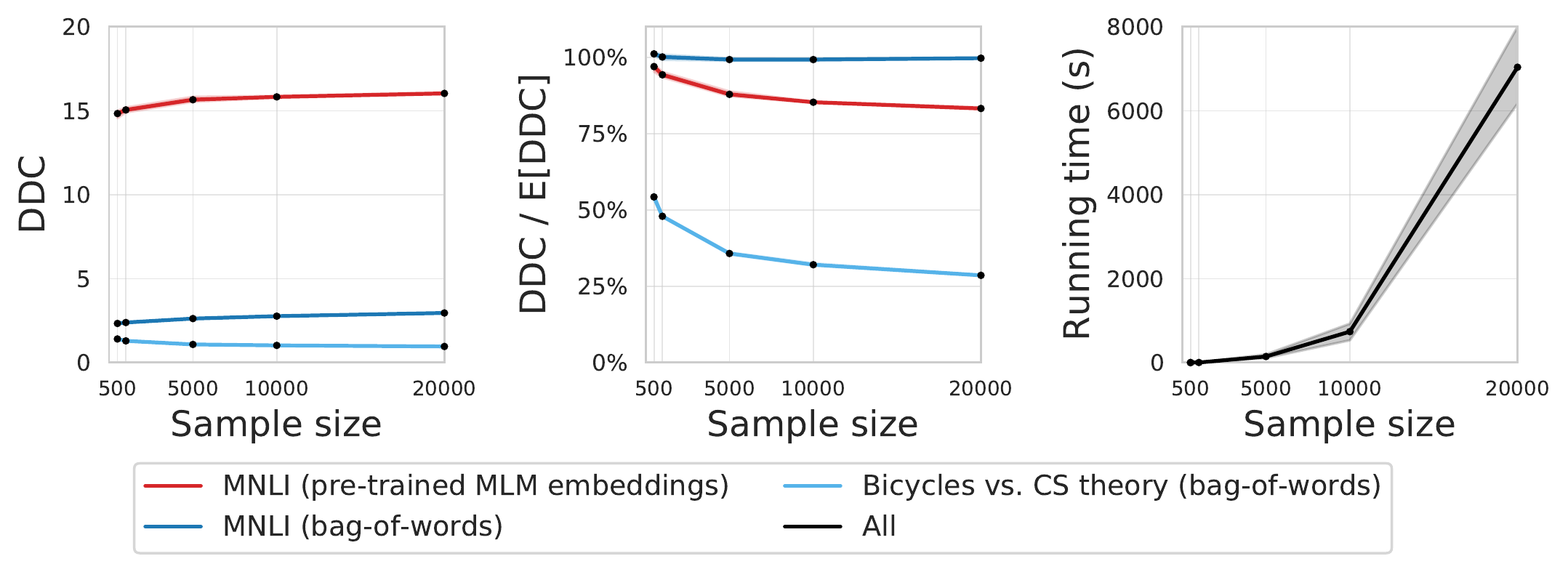}
        \caption{DDC (left) and ratio of real DDC to average random DDC (middle) are relatively stable across different-sized samples of large datasets, though larger samples incur increased 
        running time (right).}
        \label{fig:sample-size-comparison}
\end{figure}

\paragraph{Subsampling is effective for large datasets.}

Calculating DDC requires matrix operations that scale more than quadratically in the number of data points \cite{pan1999complexity}, which are prohibitive to compute exactly for large datasets.
Truncated eigendecompositions are tempting but may underestimate complexity for extremely difficult datasets; we leave exploration of truncated approximations to future work.
We recommend instead calculating DDC for a random subsample.
For experiments in this work we fix $n=20,\!000$; eigendecompositions complete in a few hours. We find that smaller values of $n$ can be much more efficient, are relatively accurate, and do not change the relative ordering for comparisons (Figure~\ref{fig:sample-size-comparison}).
We have not yet evaluated the impact of unbalanced classes.

\begin{figure}[!t]
         \centering
         \includegraphics[width=\columnwidth]{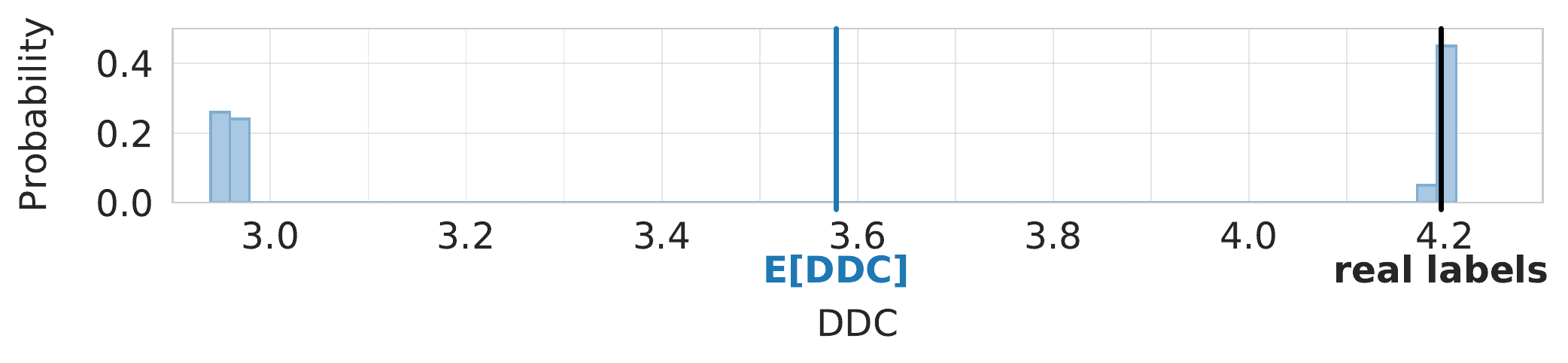}
        \caption{Two documents with nearly-identical representations entirely determine DDC of random labelings for a sample of the MNLI dataset as bags-of-words.}
        \label{fig:mnli-bow-degenerate}
\end{figure}

\paragraph{DDC distributions identify pathological cases.}
DDC can reveal potentially problematic characteristics of datasets that may not be evident under typical use.
Figure~\ref{fig:mnli-bow-degenerate} shows results for an MNLI subset with bag-of-words representations that is almost identical to the one used in Figure~\ref{fig:mnli-bow-histogram}, but the histogram of random labeling DDCs is bimodal.
This dataset contains two documents that have identical bag-of-words representations but different labels.
When these documents are randomly assigned the same label, DDC is just under 3.0, as in the other subset. But when they are assigned opposite labels (as in the real labeling) the documents by themselves are enough to increase DDC to 4.2, because they add weight to a low-ranked eigenvector with a very large inverse eigenvalue.
We see this sensitivity as a feature in a setting where we are using DDC as a diagnostic tool.
In our experiments we filter out duplicates, e.g., fewer than 0.2\% in SNLI.



\section{Experiments}

\paragraph{DDC supports comparisons between datasets and alternative labelings.}
We begin by demonstrating that our method reveals the relationship between data and labels by evaluating multiple labelings for two simple classification datasets with Stack Exchange posts represented as bags of words. Our goal is to determine the extent to which each labeling is aligned with the data.
Both datasets comprise documents from two English-language Stack Exchange communities released in the Stack Exchange Data Dump \cite{stackexchange}.
First, we choose two communities we expect to be easily distinguishable based on vocabulary: Bicycles and CS theory.
Second, we choose two communities we expect to be more difficult to distinguish: CS and CS theory.
For both datasets, we consider three valid ways to partition the data, which we also expect to be from easier to harder:
1) Community: each document is labeled with the community to which it was posted. 2) Year: each document is labeled with whether it was posted in the years 2010-2015 or 2016-2021 (both ranges inclusive). 3) AM/PM: each document is labeled with whether its timestamp records that it was posted in the hours 00:00-11:59 or 12:00-23:59.
For both datasets, we sample 20,000 documents so that each labeling assigns half the dataset to each class.
See Appendix~\ref{sec:app-pre-processing} for more details.

\begin{figure}[!t]
         \centering
         \includegraphics[width=\columnwidth]{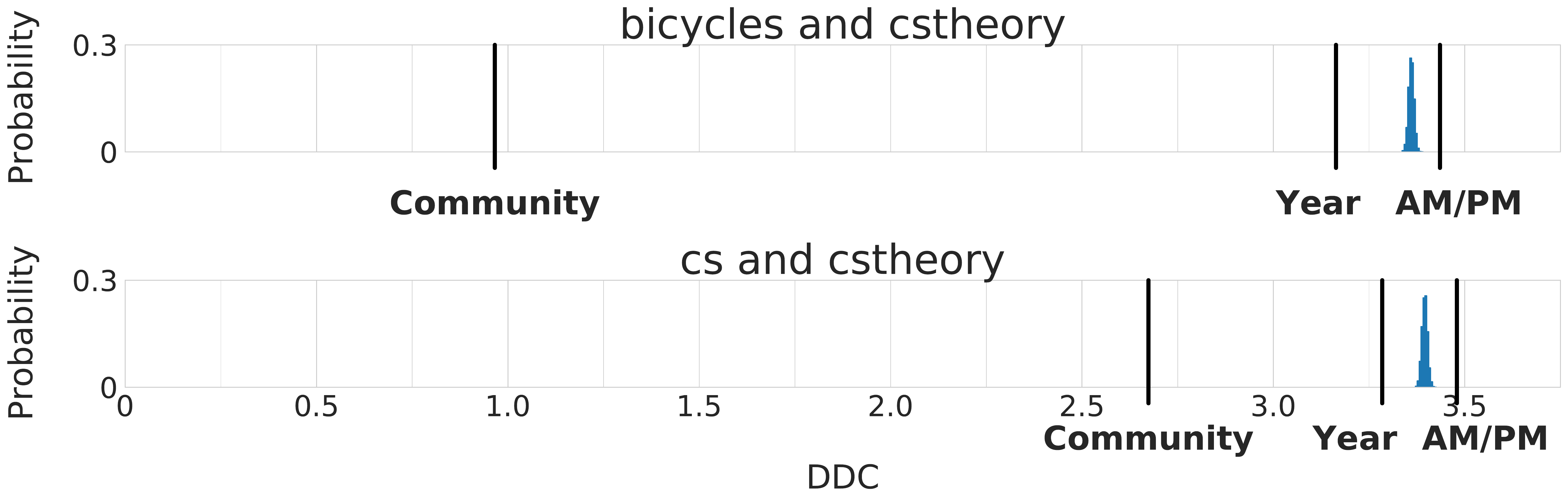}
        \caption{Comparing different valid labelings of Stack Exchange documents: labeling by community is the most different from random labelings (blue histogram). Labeling by year is still salient, but AM/PM labels are not aligned with patterns in the documents.}
        \label{fig:stackexchange}
\end{figure}

Figure~\ref{fig:stackexchange} shows DDC for both datasets using the three valid labelings and the distribution over DDC for random labelings.
As hypothesized, the DDC of the community labeling is much lower than that of random labelings for both tasks.
What's new, however, is that our method quantifies the differences in difficulty without training any classifiers: when comparing Bicycles posts to CS theory posts, the real labeling is 375 standard deviations of the random distribution below the average random DDC, but the same distance is 98 standard deviations when comparing CS and CS theory.
Surprisingly, for both tasks the AM/PM labeling in fact has higher DDC than all of the random labelings we sampled. It is more than 10 standard deviations from the average DDC of random labelings for both. We hypothesize that this labeling is unusually well balanced relative to the
actual differences in documents.

\paragraph{NLI experiment details.}


In the previous experiment we kept the data fixed and compared different labelings. Here we keep labelings fixed and compare alternative data representations.
We aim to disentangle how pre-training, fine-tuning, and the final classification step contribute to performance on natural language inference (NLI) tasks.
Our protocol for measuring data-label alignment of a dataset is: 1) choose a set of representations to compare and remove any examples that are identical under any representation, 2) for each representation: sample up to 20,000 examples from the dataset and construct the Gram matrix, 3) calculate DDC of the real labeling and DDC of random labelings, 4) compare the gap between DDC of real and random labelings across representations.
This process can be repeated across subsamples of the dataset.

We analyze training data from three English-language datasets from the GLUE benchmark \cite{wang2018glue}: MNLI  \cite{N18-1101}, QNLI \cite{rajpurkar2016squad}, and WNLI \cite{levesque2012winograd}; along with an additional fourth dataset: SNLI \cite{bowman2015large}. We use the entailment and contradiction classes. Pre-processing is in Appendix~\ref{sec:app-pre-processing}.
For baseline data representations, we use localist bag-of-words and GloVe embeddings after concatenating the two sentences in each NLI example. For GloVe, each word is represented by a static vector, and word vectors are averaged to produce one vector for the entire sentence.
We use pre-trained contextual embeddings from BERT \cite{devlin2019bert} and RoBERTa-large \cite{liu2019roberta}.
We also use contextual embeddings from RoBERTA-large fine-tuned on each of MNLI, QNLI, and SNLI. For each fine-tuning, we follow \citet{le2020adversarial}: pick a random 10\% of the training dataset, fine-tune with that sample, and then discard the sample from future analyses. This allows us to evaluate the effects of fine-tuning without trivially examining data used for fine-tuning.
For all MLM representations, each document is represented by the final hidden layer of the \texttt{[CLS]} token, as is standard.
Models were implemented using Hugging Face's Transformers library \cite{wolf-etal-2020-transformers} with NumPy \cite{harris2020array} and PyTorch \cite{paszkepytorch}.

To compare the complexity of a real labeling with the distribution of complexities from random labelings, we focus on two metrics:
1) the ratio of the real labeling's DDC to the average DDC over random labelings:  $\frac{\text{DDC}}{\Ex \left[ \text{DDC} \right]}$, and
2) the number of standard deviations of the random DDC distribution that the real DDC is from the average DDC over random labelings (z-score): $\frac{\text{DDC} - \Ex \left[ \text{DDC} \right]}{\sigma}$.
The first compares how far real DDC is from the average random DDC in terms of percentage, and the second compares how far real DDC is from the average random DDC in terms of the distribution.


\begin{figure}[!t]
         \centering
         \includegraphics[width=\columnwidth]{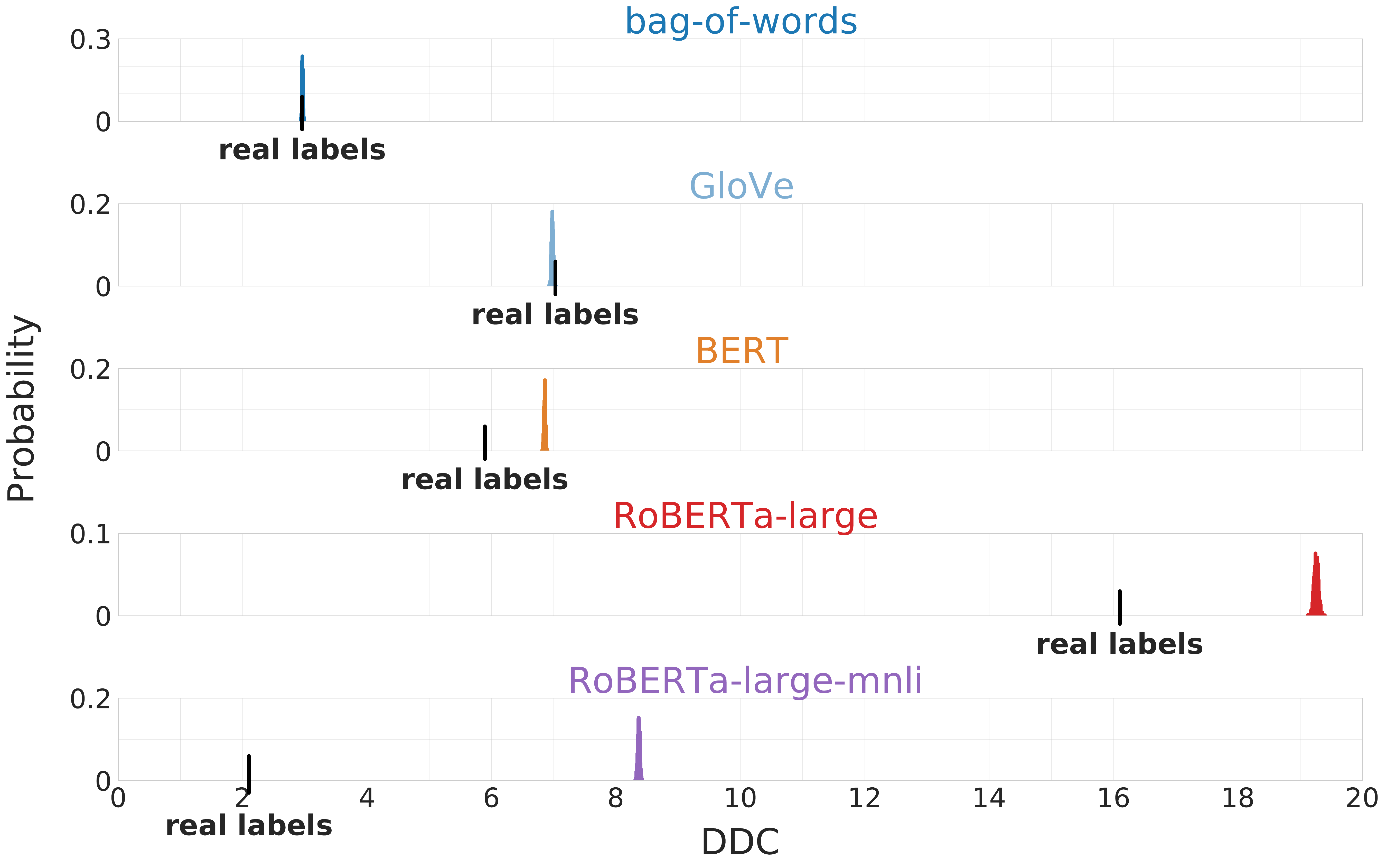}
         \begin{subfigure}[b]{0.24\linewidth}
         \centering
         \includegraphics[height=0.59in]{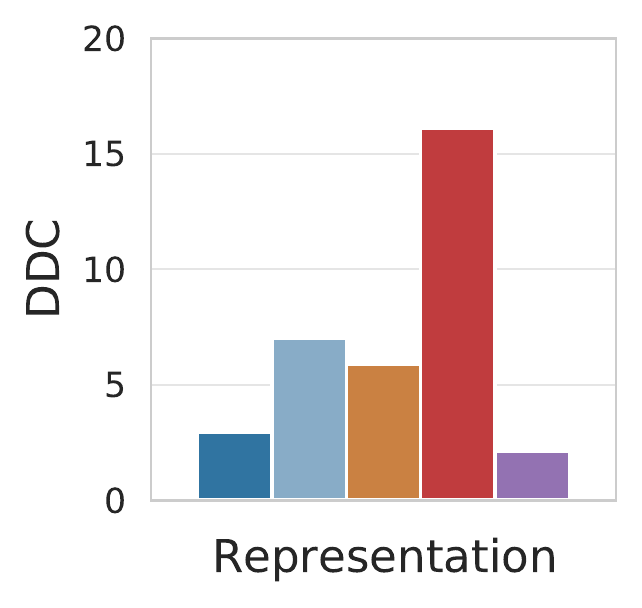} 
         \caption{}
         \label{fig:mnli-run-1-results-a}
     \end{subfigure}
     \hfill
     \begin{subfigure}[b]{0.24\linewidth}
         \centering
         \includegraphics[height=0.59in]{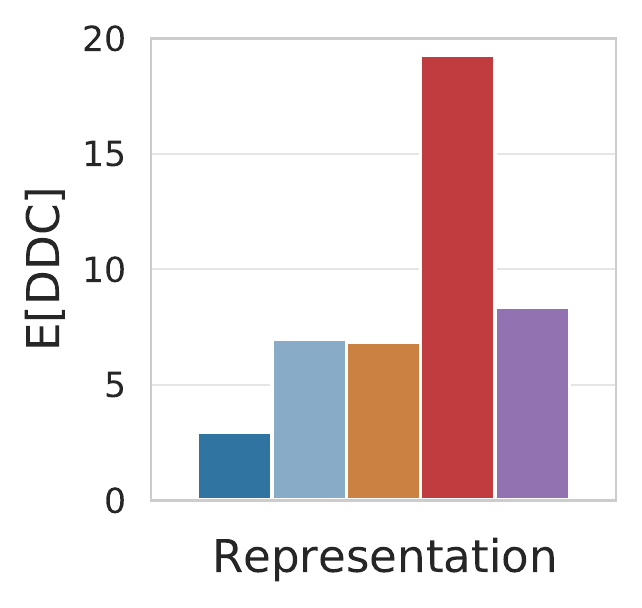} 
         \caption{}
         \label{fig:mnli-run-1-results-b}
     \end{subfigure}
     \hfill
     \begin{subfigure}[b]{0.24\linewidth}
         \centering
         \includegraphics[height=0.59in]{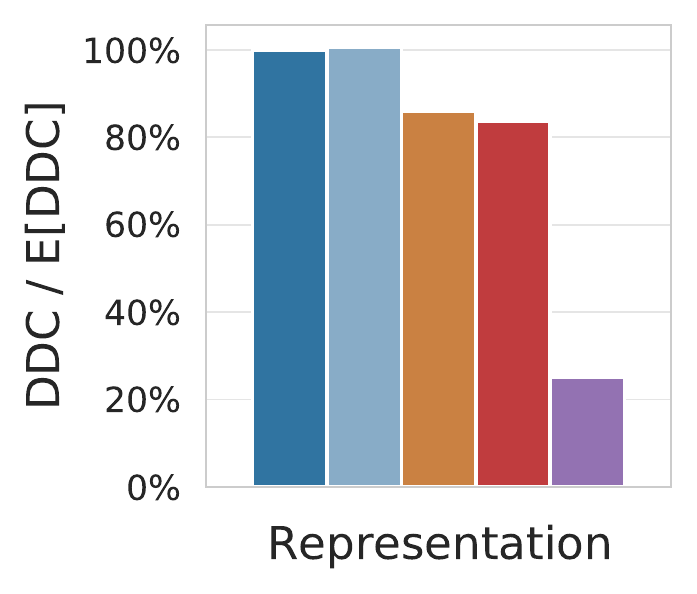} 
         \caption{}
         \label{fig:mnli-run-1-results-c}
     \end{subfigure}
     \begin{subfigure}[b]{0.24\linewidth}
         \centering
         \includegraphics[height=0.59in]{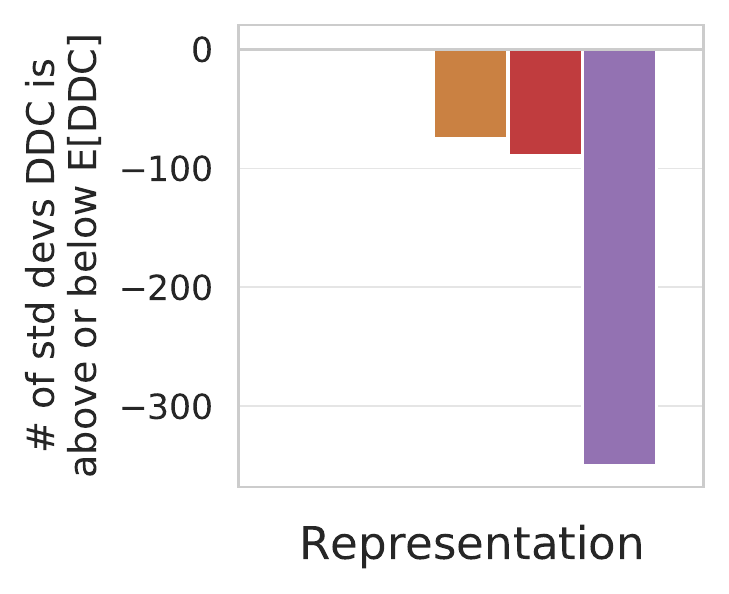} 
         \caption{}
         \label{fig:mnli-run-1-results-d}
     \end{subfigure}
        \caption{Comparisons of real and random DDCs for one subsample of MNLI: pre-trained and fine-tuned MLM representations distinguish between real and random labelings. Summary results for (a) real values, (b) mean of random values, (c) ratio of real to average random, and (d) z-score of real from average random.}
        \label{fig:mnli-all-distributions}
\end{figure}

\begin{figure*}[!t]
    \hfill
     \begin{subfigure}[t]{0.49\linewidth}
         \centering
         \includegraphics[width=\textwidth]{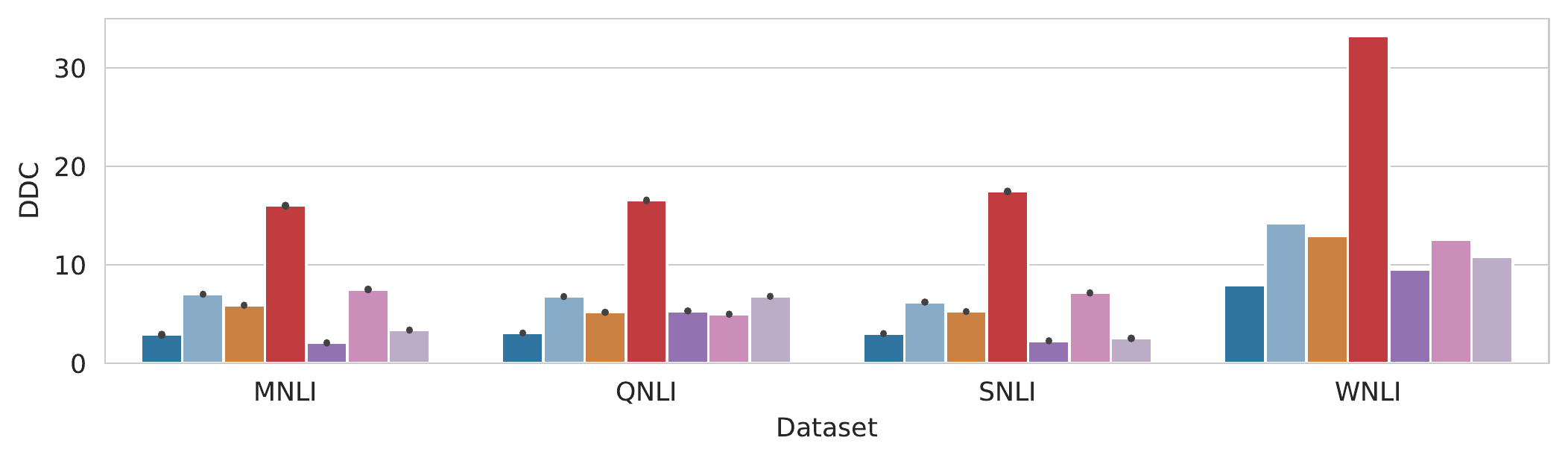} 
         \caption{DDC of real labelings varies by representation \ldots}
         \label{fig:nli-results-ddc}
     \end{subfigure}
     \hfill
     \begin{subfigure}[t]{0.49\linewidth}
         \centering
        \includegraphics[width=\textwidth]{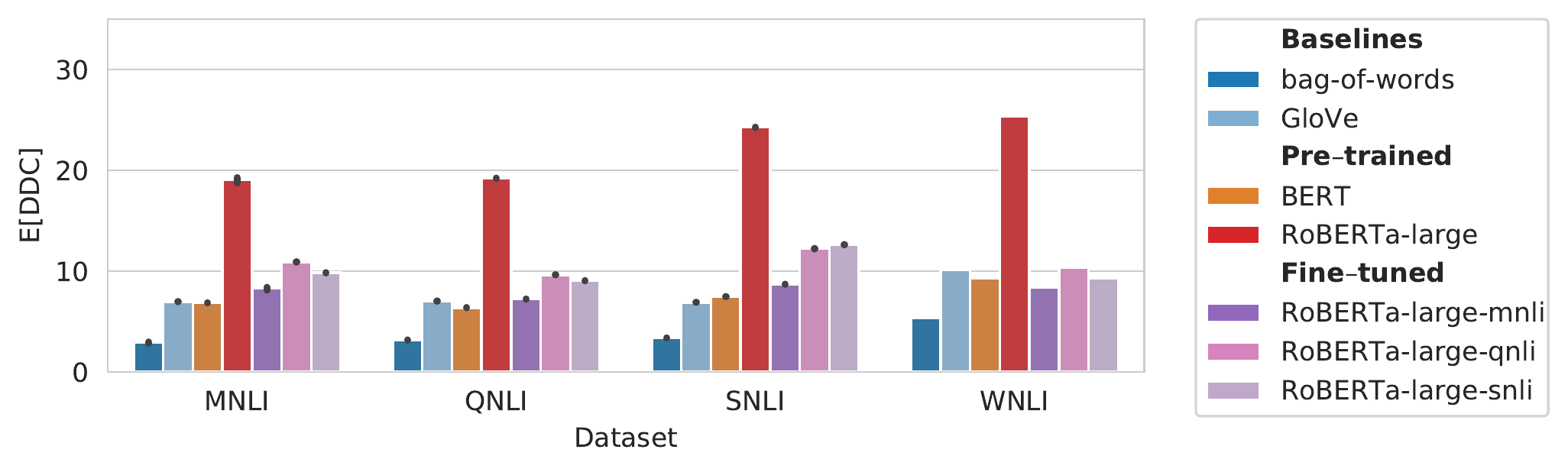}

         \caption{\ldots but average DDC for random labelings also varies in similar ways---mostly due to eigenvalue concentrations.}
         \label{fig:nli-results-ddc-random-expectation}
     \end{subfigure}
     
     \begin{subfigure}[b]{0.49\linewidth}
         \centering
         \includegraphics[width=\textwidth]{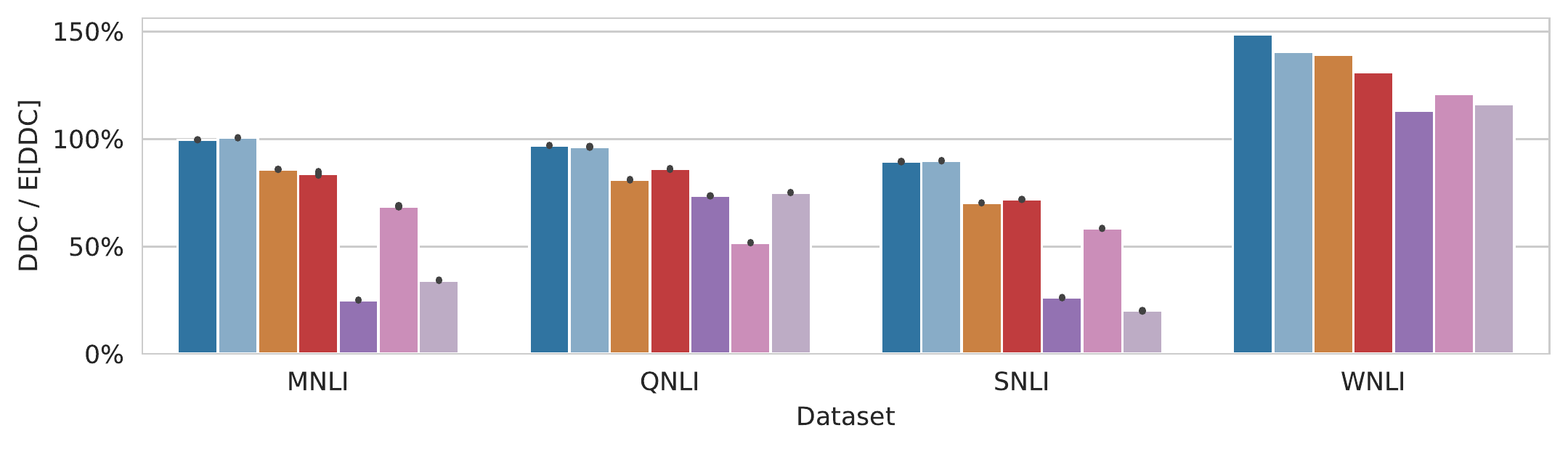} 
         \caption{The ratio of observed real DDC to average random DDC gives a more accurate perspective.}
         \label{fig:nli-results-ddc-normalized}
     \end{subfigure}
     \hfill
     \begin{subfigure}[b]{0.49\linewidth}
         \centering
         \includegraphics[width=\textwidth]{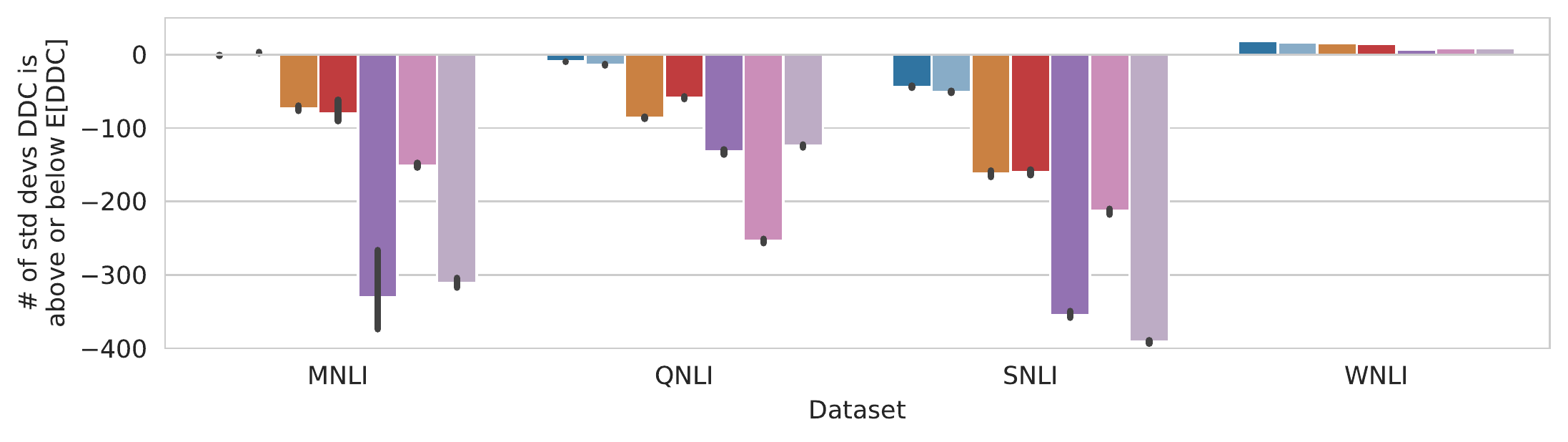} 
         \caption{The z-score shows even more difference when measuring the real labeling against the distribution of random values.}
         \label{fig:nli-results-ddc-distribution-distance}
     \end{subfigure}
     
     
        \caption{Comparisons between real and random labelings across representations and datasets. Error bars are across four replications. WNLI does not have error bars because it is small enough that it does not need to be subsampled. MNLI, QNLI, and SNLI are similar. WNLI is not well-aligned with any representation we consider.}
        \label{fig:nli-results}
\end{figure*}


\paragraph{DDC supports comparisons of representations for NLI.}
Figure~\ref{fig:mnli-all-distributions} compares the DDC of real labels to the distributions of DDCs for random labelings for a subsample of 20,000 documents from MNLI.
MNLI was specifically designed to thwart lexical approaches, so we expect bag-of-words-based methods to do poorly.
Indeed, the two baseline representations---bag-of-words and GloVe---do not distinguish between real and random labelings (the top row is identical to Figure~\ref{fig:mnli-bow-histogram} but at a different scale).
In fact, for GloVe embeddings, the real labeling is \textit{more} complex than all sampled random labelings.
For both pre-trained MLM embeddings, the value of DDC is greater than that of the bag-of-words representation, but the DDC of random labelings increases even more, leading to a much wider gap between real and random labelings and lower \textit{proportional} DDC.\footnote{We also found that pre-trained RoBERTa embeddings had near-identical results as RoBERTa-large embeddings.}
As shown in Figure~\ref{fig:pca-projection-dashboard}, while pre-trained RoBERTa-large embeddings enable training classifiers with greater accuracy, they also have a much faster drop-off in eigenvalues. DDC is therefore higher for this representation because the labels project onto ``noise'' eigenvectors with small eigenvalues.
{ But random labelings have even greater projections onto noise eigenvectors, so there is a large gap between real and random DDCs.}
Fine-tuned contextual embeddings show the largest drop in relative DDC, as well as the lowest absolute DDC for real labels.
This result is consistent with empirical dev-set accuracies for pre-trained RoBERTa-large and fine-tuned RoBERTa-large-mnli embeddings of 67.1\% and 96.6\%, respectively, but it provides additional quantitative perspective that does not rely on stochastic gradient descent algorithms.

Comparing to \textit{distributions} of DDC values for random labelings, rather than just the mean, also provides a new perspective.
If we were to just consider the ratio of real DDC to average random DDC (Figure~\ref{fig:mnli-run-1-results-c}), we would conclude, for example, that real labels for BERT have 80\% of the DDC of random labelings, on average.
But when we reconsider the distance between the real DDC and average random DDC in terms of the standard deviation of the distribution (Figure~\ref{fig:mnli-run-1-results-d}), we see that the real DDC is 71 standard deviations away from the average.
Additionally, BERT has much lower DDC for real labels than RoBERTa-large. When viewed in the context of the distribution, however, we can see that BERT and RoBERTa-large representations are nearly equal at distinguishing real from random labelings.
Surprisingly, for pre-trained and fine-tuned representations, every sampled random labeling had higher complexity than the real labeling; the probability that a random labeling has a lower DDC than the real labeling is at most $0.001$.

\paragraph{Comparisons between representations across datasets.}
In addition to comparing alternative representations for a single dataset, we can compare representations between datasets.
This method measures the ability of a neural network to produce representations that are aligned with multiple, slightly different tasks.
We repeat the previous experiment for all four NLI datasets, this time calculating complexities of  multiple $20,\!000$-sample subsets of each dataset.
For a given sample of a dataset, we compare the complexities across all data representations.
Figure~\ref{fig:nli-results} shows that the trends that we saw in one sample of MNLI are borne out across multiple samples of MNLI and across QNLI and SNLI.
We use the same baseline and pre-trained representations as before, but also include the output of RoBERTa-large models fine-tuned on a discarded subset of MNLI, QNLI, and SNLI. 

Our first result is that MNLI is unusually difficult for purely lexical methods. 
While the complexity of real labels is closest to the average random DDC for bag-of-words and GloVe representations, for QNLI and especially SNLI, we are still able to distinguish between real and random labels.
For QNLI, the real DDCs are separated a small amount from the average random DDC in absolute terms, but this is nearly 10 standard deviations. The separation is even more pronounced for SNLI, where the real complexity is separated from the average random complexity by 43 and 50 standard deviations for bag-of-words and GloVe, respectively. This result is surprising because the NLI entailment and contradiction classes should not \textit{a priori} be associated with lexical patterns.
It provides further evidence for previous findings that lexical and hypothesis-only approaches can achieve high accuracy on NLI datasets \cite{gururangan2018annotation}.

DDC for pre-trained representations appears different between BERT and RoBERTa-large, with BERT closer to GloVe, but this difference disappears when comparing to  random labelings.
While BERT and RoBERTa-large differ in their eigenvalue distributions, we cannot reliably distinguish them in terms of relative alignment with task labels.

Fine-tuned representations are significantly better aligned with labels than pre-trained embeddings.
For MNLI, QNLI, and SNLI, the pre-trained embeddings separate real and random DDCs more than the baseline representations, even when the baseline representations already achieve some separation.
As expected, fine-tuned representations distinguish the most between real and random labels.
What's new is that our method quantifies the extent of the increased alignment. Fine-tuned embeddings more than double the gap between real and random labelings beyond that of pre-trained embeddings, when measured by either ratio or standard deviations.
In addition, we see some evidence of transfer learning: representations from networks fine-tuned on one NLI dataset have greater alignment with labels on the other datasets than pre-trained representations do.
But we find that MNLI and SNLI are more able to transfer to each other, while QNLI appears significantly different.

We were surprised to find how unlike WNLI is to the other datasets we consider: even contextual embeddings are not significantly more aligned with the real labeling than the baseline representations. There are many alternative labelings that are more aligned with the structure of the data, which accords with WNLI's purpose as a hand-crafted challenge dataset \cite{levesque2012winograd}.
Our experiments suggest that fine-tuned RoBERTa's 91.3\% accuracy on WNLI \cite{liu2019roberta} comes from updating the representations during their multi-task fine-tuning and the high capacity of the classification head. Pre-training alone is not enough.

\paragraph{DDC helps guide MLM embedding choices.}
Finally, we present a case study in which data-label alignment provides guidance in modeling choices.
MLM-based embeddings have become standard in NLP, but there remain many practical questions about \textit{how} users should apply them.
For example, users may be concerned about
how the output of a network should be fed to subsequent layers.
For BERT-based models we often use the embedding of the \texttt{[CLS]} token as a single representation of a document, but \citet{miaschi-dellorletta-2020-contextual} empirically evaluate contextual embeddings from averaging hidden token representations.
In Figure~\ref{fig:embedding-comparisons}, we compare the alignment with NLI task labels of 1) the \texttt{[CLS]} token, 2) taking the mean of the final hidden layer, and 3) concatenating the final hidden layer across tokens.
We find no difference that is consistent across models and datasets, though \texttt{[CLS]} is modestly better at distinguishing real and random labelings for MNLI and SNLI. Any observed difference is less significant than the choice to fine-tune and would not change the relative order of models.
In this case, users should feel confident that there is no strong alignment advantage for NLI classification to choosing one representation over another.

\begin{figure}[!t]
         \centering
         \includegraphics[width=\columnwidth]{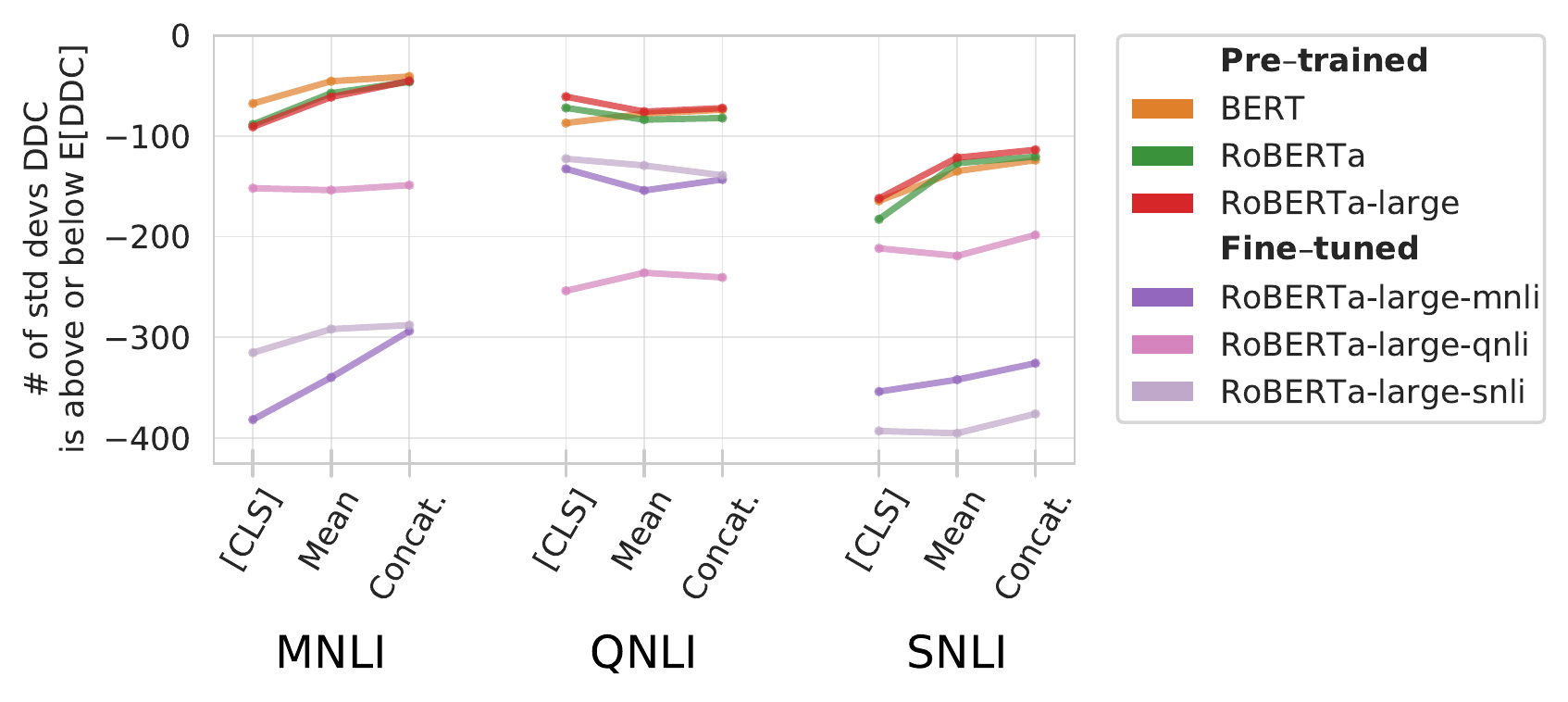}
        \caption{
        We find no consistent advantage between 1) the hidden embedding of just the \texttt{[CLS]} token, 2) averaging the final hidden embedding of all tokens, and 3) concatenating the final hidden embedding of all tokens.
        }
        \label{fig:embedding-comparisons}
\end{figure}



\section{Related Work}

We see this work as part of a more general increase in quantitative evaluations of datasets.
\citet{swayamdipta2020dataset} identify subsets of NLI datasets that are hard to learn by analyzing how a classifier's predictions change over the course of training.
\citet{sakaguchi2020winogrande} and \citet{le2020adversarial} similarly filter NLI datasets by finding examples frequently misclassified by simple linear classifiers.
Rather than find difficult examples within a dataset, we seek to understand how different data representations impact task difficulty.
Our results complement \citet{le2020adversarial} in finding that categories containing more dissimilar examples are more complex.
They find that filtering data under the RoBERTa representation leads to the most difficult reduced dataset; our work suggests this might be because contextual embeddings better differentiate the entailment and contradiction classes in NLI tasks than baseline representations.

Minimum description length (MDL)
treats a classifier as a method for data compression and has recently been used to measure the extent to which representations learned by MLMs capture linguistic structure
\cite{voita-titov-2020-information}.
\citet{perez2021rissanen} use MDL to determine if additional textual features are expected to lead to easier tasks, finding, for instance, that providing answers to sub-questions in question-answering tasks decreases dataset complexity.
In contrast, our work
investigates the relationship between different representations of fixed data and labelings.
\citet{mielke2019kind} also use information-theoretic tools (surprisal/negative log-likelihood) to empirically evaluate differences in language model effectiveness across languages.

This work also relates to prior work on evaluating the capabilities of embeddings.
Non-contextual dense vector word representations encode syntax and semantics \cite{pennington2014glove, mikolov2013} as well as world information like gender biases \cite{bolukbasi2016man}.
Linguistic probing and comparison of MLM layers and representations has identified specific capabilities of MLMs \cite{tenney2019you, liu-etal-2019-linguistic, rogers2020primer}.
Many works have identified and proposed antidotes to the anisotropy of non-contextual word embeddings \cite{arora2016simple, mimno2017strange, mu2017all, li-etal-2020-sentence}. \citet{conneau2020emerging} use a kernel approach to compare embeddings learned by MLMs pre-trained on different languages.
While analyzing the geometry of contextual embedding vectors
remains an active line of work \cite{reif2019visualizing, ethayarajh-2019-contextual}, we instead analyze the relationship between embeddings and the labels of downstream tasks.


\section{Conclusion}

We present a method for quantifying the alignment between a representation of data and a set of associated labels.
We argue that the difficulty of a dataset is a function of the alignment between the chosen representation and a labeling, relative to the distribution of alternative labelings, and we demonstrate how this method can be used to compare different text representations for classification.
{ We used NLI datasets as well-understood case studies in order to demonstrate our method, which replicates results from less general methods that were surprising when introduced, such as \citet{gururangan2018annotation} on lexical mutual information.
Our method supplements traditional held-out test set accuracy: while accuracy answers \textit{which} representations enable high performance for a task, our approach offers more explanation of \textit{why}.}

{ We hope that future work can study novel datasets and settings.}
Our method can be readily applied to new datasets, and it could especially be used to quantify the difficulty of  adversarially constructed datasets like WinoGrande \cite{sakaguchi2020winogrande} and Adversarial NLI \cite{nie2020adversarial}.
Our method could be used to measure data-label alignment while changing the information present in each document, as in the recent work of \citet{perez2021rissanen}.
The method can also be extended to both other classification models by analyzing Gram matrices produced by different similarity kernels
and to multi-class, imbalanced, and noisy settings where uniformly-random labelings are not as applicable.
More theoretical work could provide further generalization guarantees.

Our method provides a new lens that can be used in multiple ways.
NLP practitioners want to design models that achieve high accuracy on specific tasks, and our method can identify which representation most aligns with the task's labeling and whether certain processing steps are useful.
Dataset designers, on the other hand, often seek to provide challenging datasets in order to spur new modeling advances \cite{wang2018glue, wang2019superglue, sakaguchi2020winogrande}. Our method helps diagnose when current text representations do not capture the variation in a dataset, as we showed for the case of WNLI, necessitating richer embeddings. It can also indicate when datasets are not robust to patterns in existing embeddings, as we found with SNLI and QNLI aligning with baseline lexical features.

Finally, while empirical exploration has been an effective strategy in NLP, better theoretical analysis may reveal simpler yet more powerful and more explainable solutions \cite{saunshi2020mathematical}.
Applying theory-based analysis to \textit{representations} rather than to NLP models as a whole offers a way to benefit immediately from such perspectives without requiring full theoretical analyses of deep networks.




\section*{Acknowledgments}

We thank Maria Antoniak, Katherine Lee, Rosamond Thalken, Melanie Walsh, and Matthew Wilkens for thoughtful feedback. We also thank Benjamin Hoffman for fruitful discussions of analytic approaches. 
This work was supported by NSF \#1652536.

\section*{Ethical Considerations}

We do not anticipate any harms from our use of publicly available Stack Exchange datasets and NLI datasets.
While this work analyzes the text representations produced by large masked language models, we do not anticipate any harms from the method presented in this work.
We fine-tune these large models in order to calibrate our method, but our method can be used on already-trained models and does not necessitate additional training.
We believe that more analytical and interpretive work like ours can better guide empirical computation-intensive research.

\bibliography{emnlp2021}
\bibliographystyle{acl_natbib}

\clearpage

\appendix

\section{Dataset Pre-Processing}
\label{sec:app-pre-processing}

\subsection{Natural Language Inference}

\begin{table}[b!]
\resizebox{\linewidth}{!}{%
\begin{tabular}{l S[table-format=6.0] S[table-format=5.0] S[table-format=6.0] S[table-format=6.0]}
\toprule
 & MNLI & QNLI & SNLI & WNLI\\
\hline
Duplicates removed & 721 & 198 & 748 & 4\\
10\% used in fine-tuning & 39270 & 10474 & 54936 & n/a \\
Unique docs excluded & 39922 & 10653 & 55586 & 4 \\
Total examples used & 234122 & 92649 & 330080 & 631 \\
\bottomrule
\end{tabular}
}
\caption{Summary statistics for NLI datasets.}
\label{table:nli-datasets}
\end{table}

We compare representations of English-language natural language inference datasets from the GLUE benchmark \cite{wang2018glue}: MNLI  \cite{N18-1101}, QNLI \cite{rajpurkar2016squad}, and WNLI \cite{levesque2012winograd}. We also use SNLI \cite{bowman2015large}. Table~\ref{table:nli-datasets} shows summary statistics. Starting from documents labeled entailment or contradiction, we exclude any documents that are string identical (surprisingly, these datasets do contain a few duplicates) and any that are identical under any representation that we consider. Practically, this means excluding documents that have identical bag-of-words representations (because of different word order) or identical GloVe embeddings (because of words/numbers not in the GloVe vocabulary). We fine-tune on 10\% of each training set and exclude those documents as well, as in \citet{le2020adversarial}.

\subsection{Stack Exchange}

\begin{table}[b!]
\resizebox{\linewidth}{!}{%
\begin{tabular}{l l S[table-format=6.0] S[table-format=5.0] S[table-format=6.0]}
\toprule
Year & AM/PM & \text{Bicycles} & \text{CS} & \text{CS Theory}\\
\hline
$[2010, 2015]$ & $[00:00-11:59]$ & 9872 & 11596 & 7661\\
$[2010, 2015]$ & $[12:00-23:59]$ & 17038 & 17371 & 11888\\
$[2016, 2021]$ & $[00:00-11:59]$ & 11370 & 22264 & 2890\\
$[2016, 2021]$ & $[12:00-23:59]$ & 18415 & 34408 & 4434\\[1mm]
\multicolumn{2}{l}{Total across all years and AM/PM} & 56695 & 85639 & 26873\\
\bottomrule
\end{tabular}
}
\caption{Label breakdowns for the three Stack Exchange communities we consider.}
\label{table:stack-exchange}
\end{table}

We construct text classification tasks from English-language posts on Stack Exchange communities \cite{stackexchange}. We construct two datasets, the first with posts from Bicycles and CS Theory and the second with posts from CS and CS Theory. Table~\ref{table:stack-exchange} shows summary statistics for each community. For each dataset, we sample $2,\!500$ posts from each community-year-AM/PM combination, resulting in $10,\!000$ documents from each community.
This results in three balanced classification tasks for each of the datasets, allowing us to compare their difficulty without label imbalance.

\section{Estimating $\Ex[\text{DDC}]$}
\label{sec:app-bound}

The expectation of DDC over uniformly random labelings can be accurately estimated by averaging the DDC of sampled labelings. The following claim shows that the number of samples required is determined by the difference between the inverses of the largest and smallest eigenvalues of the Gram matrix. 
Recall that $\mathbf{H}^\infty$ is this Gram matrix with number of examples $n$.
Note that Jensen's inequality can be used to upper-bound the expectation by $\sqrt{\left( \frac{2}{n} \right) \Tr \left[ (\mathbf{H}^\infty)^{-1} \right] }$, but it is not readily apparent how to compute the expectation exactly.

\begin{claim}
To estimate the expected DDC over random labelings to within $\varepsilon$ with probability at least $1 - \delta$, averaging DDC of the following number of sampled uniformly random labelings is sufficient:
\[m \geq \frac{\Delta^2}{2 \varepsilon^2}\ln \left( \frac{2}{\delta} \right)\]
where $\Delta$ is the difference between the maximum and minimum DDC values.
\end{claim}

\begin{proof}
First, note that the DDC of a random labeling is bounded. Using the eigendecomposition interpretation of DDC, we find that the maximum DDC can be no greater than when the label vector $\mathbf{y}$ projects entirely on the Gram matrix's final eigenvector with smallest eigenvalue $\lambda_\text{min}$: 
\begin{align}
\DDC_\text{max} &\leq \sqrt{\frac{2 \Vert \mathbf{y} \Vert \Vert \mathbf{y} \Vert}{\lambda_{\text{min}}} \cdot \frac{1}{n} }\\
&= \sqrt{\frac{2 \sqrt{n} \sqrt{n}}{\lambda_{\text{min}}} \cdot \frac{1}{n} }
= \sqrt{\frac{2}{\lambda_\text{min}}}
\end{align}
Similarly, the minimum DDC can be no less than when the label vector projects entirely on the Gram matrix's initial eigenvector with largest eigenvalue: $\DDC_\text{min} \geq \sqrt{\frac{2}{\lambda_\text{max}}}$. Let $\Delta$ be the magnitude of this difference:
\begin{align}
\Delta &= \vert \DDC_\text{max} - \DDC_\text{min} \vert\\
&\leq \sqrt{\frac{2}{\lambda_\text{min}}} - \sqrt{\frac{2}{\lambda_\text{max}}}
\end{align}

Let $X_i$ be the DDC of the $i^\text{th}$ random labeling, and let $\bar{X} = \frac{1}{m} \left( X_1 + X_2 + \ldots + X_m \right)$ be the empirical mean of the first $m$ random DDCs. $\Ex\!\left[ \bar{X} \right]$ is then the true expectation of DDC over random labelings. Because $X_i \in \left[ \DDC_\text{min}, \DDC_\text{max} \right]$ is a bounded random variable, $\frac{1}{m} X_i$ is also bounded: $\frac{1}{m} X_i \in \left[ \frac{\DDC_\text{min}}{m}, \frac{\DDC_\text{max}}{m} \right]$ with difference between maximum and minimum values no greater than $\frac{\Delta}{m}$.
By the Hoeffding bound:
\begin{align}
\Pr \left[ \left| \bar{X} - \Ex \left[ \bar{X} \right] \right| \geq \varepsilon \right] \leq 2 \exp \left( \frac{-2 \varepsilon^2}{\sum\limits_{i \in [m]} \left( \frac{\Delta}{m} \right) ^2} \right)
\end{align}
Setting the right-hand side to be at most $\delta$ and solving for $m$ yields the stated bound on the number of required samples.
\end{proof}

\paragraph{How unlikely is the real labeling's DDC?}
Additionally, the probability that a random labeling has as low a complexity as the real labeling is given by the distribution function of random-label DDCs evaluated at the real DDC, denoted $F(\DDC)$.\footnote{Note that this isn't the same as \textit{any} labeling having as low a complexity, which could be found by union bounding over random labelings.} Let $\hat{F}(\DDC)$ be the empirical distribution function evaluated at the real DDC: the fraction of sampled random labelings with complexities less than that of the real labeling.
\citet{wasserman2013all} shows that the DKW Inequality can be used to bound the value of $F(\DDC)$ for the above number of samples $m$. With probability $1 - \delta$, $F(\DDC) \in [ \hat{F}(\DDC) - \gamma, \hat{F}(\DDC) + \gamma ]$ where:
\[\gamma = \sqrt{\frac{1}{2m} \ln \left( \frac{2}{\delta} \right)}\].

\section{Experimental Setup and Computing Infrastructure}

\paragraph{Classification results in Figure 1.}
We used two-layer fully-connected ReLU networks with 10,000 hidden units, as in \citet{arora2019fine}. We trained each network to convergence on the pictured 10,000 document-subset and evaluated on the standard dev split. We used an Intel Xeon CPU @ 2.00GHz with 27.3GB of RAM and an NVIDIA Tesla V100-SXM2 GPU.
Running times for reported runs varied between datasets, from 7.9 seconds for MNIST to 190 seconds for MNLI represented as bags-of-words. We report the best dev accuracy for learning rate $\text{lr} \in \{10^{-3}, 10^{-4}, 10^{-5}, 10^{-6}\}$.

\paragraph{Fine-tuning RoBERTa-large.}
We fine-tuned RoBERTa-large for MNLI, QNLI, and SNLI on ten percent of each dataset's training data. We used an Intel Core i7-5820K CPU @ 3.30GHz with two NVIDIA GeForce GTX TITAN X GPUs and 64 GB of RAM. We train for three epochs and hyperparameter search over initial learning rate $\textsf{lr} \in \{1e\text{--}5, 2e\text{--}5, 3e\text{--}5\}$, as in \citet{liu2019roberta}, with a fixed per-GPU batch-size of 16. For the fine-tuned contextual embeddings, we use the networks which attained highest accuracy on the dev splits. In all three cases, this was the network with $\textsf{lr} = 2e\text{--}5$. Training these models took 6076.9 seconds for MNLI, 1624.58 seconds for QNLI, and 8459.5 seconds for SNLI.

\paragraph{Comparing NLI representations.}
We compare baseline representations (bag-of-words and GloVe embeddings) and MLM contextual embeddings. For the MLM embeddings, we use BERT ($110 \times 10^6$ parameters), RoBERTa ($125 \times 10^6$ parameters), and RoBERTa-large ($355 \times 10^6$ parameters).

For calculating contextual embeddings from pre-trained and check-pointed fine-tuned MLMs, we used an Intel Core i7-5820K CPU @ 3.30GHz with an NVIDIA GeForce GTX TITAN X and 64 GB of RAM.
For eigendecompositions, we used an Intel Xeon Gold 6134 CPU @ 3.20GHz with 528 GB of RAM.
It took $215.44 \pm 98.84$ seconds to calculate contextual embeddings, $7530.95 \pm 2399.05$ seconds to construct the Gram matrix and perform the eigendecomposition, and $2954.48 \pm 1449.11$ seconds to sample random labels.

\end{document}